\documentclass{article}

 \usepackage[final]{neurips_2024}

\usepackage[utf8]{inputenc} % allow utf-8 input
\usepackage[T1]{fontenc}    % use 8-bit T1 fonts
\usepackage{hyperref}       % hyperlinks
\usepackage{url}            % simple URL typesetting
\usepackage{booktabs}       % professional-quality tables
\usepackage{amsfonts}       % blackboard math symbols
\usepackage{nicefrac}       % compact symbols for 1/2, etc.
\usepackage{microtype}      % microtypography
\usepackage{xcolor}         % colors

\usepackage{amsmath}
\usepackage{bm}
\usepackage{mathtools}
\usepackage{amsthm}
\newtheorem{definition}{Definition}
\newtheorem{theorem}{Theorem}
\usepackage{physics}
\usepackage{cases}

\usepackage{amssymb}
\usepackage{xcolor}
\definecolor{lightergray}{gray}{0.9}
\usepackage{mdframed}
\mdfsetup{skipabove=10pt,skipbelow=2pt}

\title{Predictive Coding Graphs are a Superset of Feedforward Neural Networks}

\author{%
  Bj\"orn van Zwol
    \\
  Department of Information \& Computing Sciences, Utrecht University\\
  Utrecht, The Netherlands\\
  \texttt{bjornvanzwol@gmail.com} \\
}

\begin{document}

\maketitle

\begin{abstract}
Predictive coding graphs (PCGs) are a recently introduced generalization to predictive coding networks, a neuroscience-inspired probabilistic latent variable model. Here, we prove how PCGs define a mathematical superset of feedforward artificial neural networks (multilayer perceptrons). This positions PCNs more strongly within contemporary machine learning (ML), and reinforces earlier proposals to study the use of non-hierarchical neural networks for ML tasks, and more generally the notion of topology in neural networks.
\end{abstract}

\begin{figure}[h!]
\centering
\includegraphics[width=0.85\textwidth]{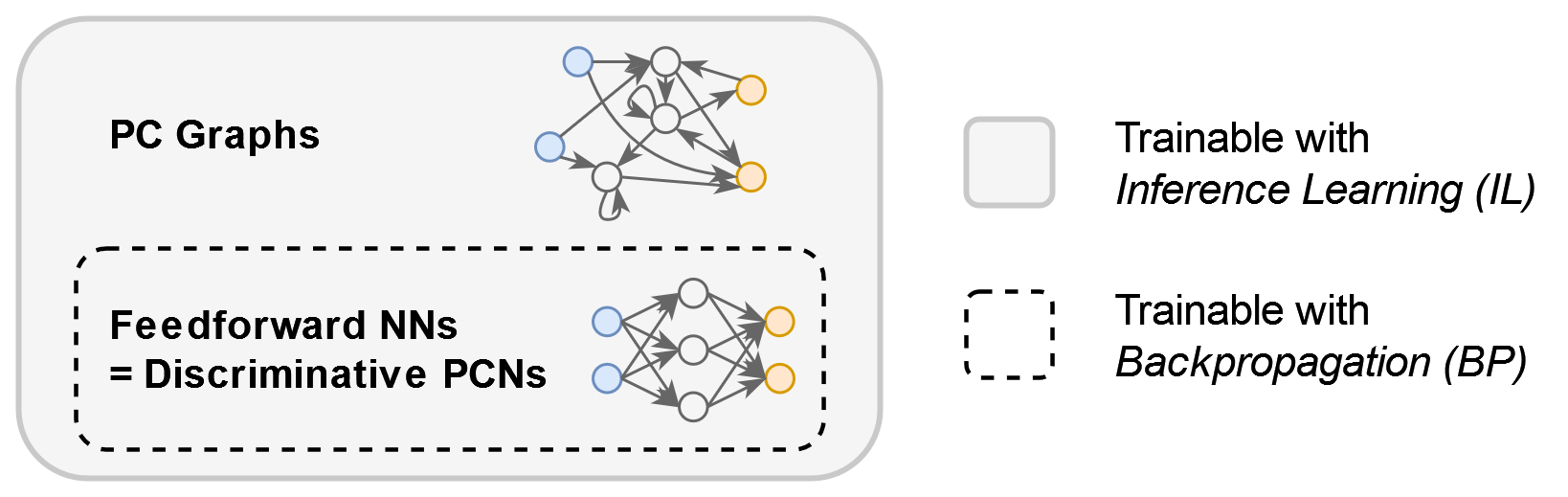}
  \caption{PCGs trained with IL generalize the structure of FNNs to arbitrary graphs, including loops and non-hierarchical structures, which are not trainable using BP.}
\label{fig:overview}
\end{figure}
\section{Introduction}
Predictive coding networks (PCNs), based on the neuroscientific framework predictive coding, have recently gained attention in machine learning (ML) \cite{van_zwol_predictive_2024,salvatori_brain-inspired_2023} for their increased biological plausibility compared to backpropagation (BP) \cite{rumelhart_learning_1986,golkar_constrained_2022,song_inferring_2024}, their parallelizability \cite{salvatori_stable_2024} and potential for probabilistic/generative modeling \cite{marino_predictive_2022,ororbia_neural_2022,zahid_sample_2024}. These networks can also be extended to arbitrary topologies, called predictive coding graphs (PCGs) \cite{salvatori_learning_2022}. 

When applied to supervised learning, standard (hierarchical) PCNs {have the same outputs} as traditional feedforward neural networks (FNNs, or multilayer perceptrons) during {testing} (i.e. inference) \cite{whittington_approximation_2017,frieder_non-convergence_2022}. This enables direct comparison of BP with PCN's training phase, also called inference learning (IL) -- which has been the focus of most recent work on PCNs in ML \cite{song_can_2020,rosenbaum_relationship_2022,salvatori_reverse_2022,innocenti_understanding_2023}.
 In this work we revisit the {testing} (inference) phase, and formally prove that {PCGs are a superset of feedforward neural networks}. This follows from the combination of two insights.

The first is that discriminative PCNs are \textit{equivalent} to FNNs during testing, for which we provide a simple proof. This is a stronger variant of related statements in the literature which show how PCNs \textit{converge to} an FNNs' computations \cite{whittington_approximation_2017,frieder_non-convergence_2022}. This subtle but important reframing enables the formal statement that the \textit{universal approximation theorem} (UAT) \cite{cybenko_approximation_1989} holds also for PCNs. The UAT is a foundational result which historically provided strong theoretical justification for using FNNs \cite{prince_understanding_2023}. {Its applicability to PCNs however, although generally believed in the PC community, lacked a principled formal proof. To the author's best knowledge, this is provided here for the first time, by virtue of the PCN-FNN testing equivalence.}

Our second insight is that PCGs define a mathematical superset of PCNs, which we also prove formally. {In ref. \cite{salvatori_learning_2022} it was well-illustrated how hierarchical structures could be obtained with PCGs by masking weights. However, precisely how such masked PCGs relate to PCNs and FNNs was unclear (i.e. their cost functions and dynamics), since a detailed formal analysis was lacking. Here, we prove that a PCG with a particular choice of weight matrix is exactly equivalent to a PCN, both in structure and dynamics.}

Combining these two insights leads to the non-trivial conclusion that PCGs should be understood as a structural superset of FNNs that includes FNNs as a special case, as visualized in fig. \ref{fig:overview}. {We believe this result substantially clarifies the relation between PC and traditional neural networks, as recently argued in \cite{van_zwol_predictive_2024}.} Moreover, understanding PCGs in this way is potentially very interesting given the topological nature of important advances in ML in the past, such as residual/skip connections \cite{he_deep_2016,zagoruyko_wide_2016,orhan_skip_2018}. Thus, we underscore the point made by \cite{salvatori_learning_2022}: PCGs are a promising framework for studying the importance of {network topology} in ML tasks. Finally, our work also highlights the value of mathematical studies of PCNs, complementary to experimental approaches more commonly found in the literature.

% Like feedforward neural networks (FNNs), PCNs involve an activity rule and a learning rule, where training involves both and testing only the former. 

% Second, we prove that PCGs are mathematical supersets of hierarchical PCNs.
% because one can prove that the solution to PCN's activity rule is exactly equivalent to FNN's activity rule. This result has already been used e.g. by , but we provide a proof for completeness, formally justifying the statement that FNNs are equivalent to PCNs during testing.
\section{Results}
We first prove how PCNs are FNNs during testing, followed by how PCNs are subsets of PCGs. The problem setup is as follows: we are given a dataset of $N$ labeled samples $\{\bm{x}^{(n)},\bm{y}^{(n)}\}_{n=1}^N$ split into a training set and a test set, where $\bm{x}^{(n)}\in \mathbb{R}^{n_x}$ is a datapoint with dimension $n_x$ and $\bm{y}^{(n)} \in \mathbb{R}^{n_y}$ its corresponding label with dimension $n_y$.
\subsection{PCNs are FNNs during testing}
We provide definitions and state the first result. We separate a neural network's \emph{activity rule} (i.e. changing activity of nodes, operating during training and testing) and \emph{learning rule} (changing weights, operating only during training) \cite{mackay_information_2003}.

% \begin{mdframed}
    \begin{definition}
    \setcounter{footnote}{1}% Reset the footnote counter
A FNN is defined by a set of nodes $a_i^\ell\in\mathbb{R}$ in layers with $n_\ell\in\mathbb{N}^+$ nodes, with $0\leq\ell\leq L$, $1\leq i(\ell) \leq n_\ell$ and $n_0=n_x$.\footnote{The index $i$ for FNNs/PCNs is always defined with respect to layer $\ell$, but we leave out this dependence henceforth for simplicity.} Its \emph{activity rule} is:
\begin{equation}
\begin{aligned}
0< \ell\leq L,\,\forall i:\quad &a^\ell_i=f\Big(\sum_j w^{\ell-1}_{ij}a^{\ell-1}_j\Big),
\end{aligned}
   \label{eq:activity_FNN}
\end{equation}
with $\forall i,\,a^0_i=x_i$, and where $w_{ij}^\ell \in \mathbb{R},\,0\leq\ell<L$ are the weights, and $f$ is a non-linear, element-wise activation function.
\end{definition}
% \end{mdframed}
Note that for FNNs, the learning rule is usually defined separate from the network definition, i.e. the canonical training algorithm backpropagation (BP) defines the learning rule:
$\Delta w_{ij}^\ell \propto {\partial\mathcal{L}}/{\partial w_{ij}^\ell}$
where $\mathcal{L}$ is some loss function. For PCNs, however, the learning rule is typically understood to be part of the network definition itself. 
% \begin{mdframed}
\begin{definition}
A PCN is defined by a set of nodes $a_i^\ell\in\mathbb{R}$ in layers with $n_\ell\in\mathbb{N}^+$ nodes, with $0\leq\ell\leq L$, $1\leq i \leq n_\ell$ and $n_0=n_x$. Its \emph{energy} is $E_N=\sum_{\ell=1}^L\sum_{i=1}^{n_{\ell}}(\epsilon^\ell_i)^2$, with $\epsilon^\ell_i = a^\ell_i-\mu^\ell_i$ and  $\mu^\ell_i=f\Big(\sum_{j=1}^{n_{\ell-1}} w^{\ell-1}_{ij}a^{\ell-1}_j\Big)$, where $w_{ij}^\ell \in \mathbb{R},\, 0\leq\ell<L$ are the weights, and $f$ is a non-linear element-wise activation function. Its \emph{activity rule} is (using hats for minimized values):
\begin{align}
% \ell=0,\,\forall i:\quad&a^0_i=x_i \\
% \ell=L,\,\forall i:\quad & a^L_i=y_i\\
0<\ell<L,\,\forall i:\quad&\hat{a}_i^\ell=\underset{{a_i^\ell}}{\operatorname{argmin}}\,E_N & \emph{(training)}    \label{eq:activity_PCN_training} \\
0< \ell\leq L,\,\forall i:\quad&\hat{a}_i^\ell=\underset{{a_i^\ell}}{\operatorname{argmin}}\,E_N & \emph{(testing),}    \label{eq:activity_PCN_testing}
\end{align}
with $\forall i,\,a^0_i=x_i$, and during training $\forall i,a^L_i=y_i$. The \emph{learning rule} is:
\begin{equation}
\begin{aligned}
%  \ell=0,\,\forall i:\quad & a^0_i=x_i\\
% \ell=L,\,\forall i:\quad & a^L_i=y_i\\
0\leq\ell<L,\,\forall i,j:\quad &\hat{w}_{ij}^\ell=\underset{w_{ij}^\ell}{\operatorname{argmin}}\,E_N  & \emph{(training),}
\end{aligned}
\end{equation}
with $\forall i: a^0_i=x_i,\,a^L_i=y_i$.
\end{definition}
% \end{mdframed}
% \begin{mdframed}

\begin{theorem}\label{thm:thm1}
During testing, a PCN {is equivalent to} an FNN.
\end{theorem}
% \end{mdframed}
We prove this in appendix \ref{sec:proof_thm1}, and provide an intuition here. Testing means that only the activity rule is relevant, and that $\forall i:a^0_i=x_i$. Equivalence between an FNN and PCN then means that \eqref{eq:activity_FNN} is equivalent to \eqref{eq:activity_PCN_testing}:
\begin{equation}
   % \forall \ell,i:\quad \hat{a}_i^\ell=\underset{{a_i^\ell}}{\operatorname{argmin}}\,E, \iff  a^\ell_i=\sum_j w^{\ell-1}_{ij}f(a^{\ell-1}_j).
     \ell>0,\,\forall i:\quad \hat{a}_i^\ell=\underset{{a_i^\ell}}{\operatorname{argmin}}\,E_N \iff  a^\ell_i=f\Big(\sum_j w^{\ell-1}_{ij}a^{\ell-1}_j\Big).
   \label{eq:equivalence}
\end{equation}
% The PCN energy is:
% \begin{equation}
%     E_N  
%     % \frac{1}{2} \sum_{\ell=1}^L\sum_{i=1}^{n_{\ell}}  (\epsilon_i^\ell)^2 
%     =\frac{1}{2} \sum_{\ell=1}^L\sum_{i=1}^{n_{\ell}}\Big[ a_i^\ell- f\Big(\sum_{j=1}^{n_{\ell-1}} w^{\ell-1}_{ij} a^{\ell-1}_{j}\Big)\Big]^2.
% \end{equation}
% Observe that the lowest layer $\ell=0$ has no error term.
To prove this, start from the left hand side, and find the minimum of $E_N$ by taking derivatives w.r.t. ${a}_i^\ell$, and setting to zero. This yields the following system of equations (cf. appendix \ref{sec:appendix_derivations}):
\begin{numcases}{ \pdv{E_N}{a_i^\ell}=}
\epsilon^\ell_i-\sum_{j=1}^{n_{\ell}} w^\ell_{ji}\epsilon_j^{\ell+1}f'\Big(\sum_m w^{\ell}_{jm}a^\ell_m\Big)=0 &$\text{if}\quad 1\leq\ell < L$ \label{eq:SOE_compact_1} \\
\epsilon^L_i=0  &$\text{if}\quad \ell=L$ \label{eq:SOE_compact_2}
\end{numcases}
This may be solved by seeing that since $\epsilon_i^L=0$ by \eqref{eq:SOE_compact_2}, $\epsilon_i^{L-1}=0$ by  \eqref{eq:SOE_compact_1} -- an argument which can be continued for each layer until $\epsilon_i^1=0$. Then, by definition of $\epsilon^\ell_i$, one has the right hand side of \ref{eq:equivalence}. This can be formalized by backwards induction.

Our proof is simpler than similar proofs in the literature that we are aware of \cite{song_predictive_2021,salvatori_reverse_2022,frieder_non-convergence_2022}, since we employ a more general definition of the activity rule. Instead of \eqref{eq:activity_PCN_training}, \eqref{eq:activity_PCN_testing}, these works define the activity rule of PCNs as $\Delta {a}_i \propto {\partial E}/{\partial a_i}$, which is shown to converge to \eqref{eq:activity_FNN}. This however {assumes} gradient-based dynamics, which we argue is unnecessary using the more principled definition of IL as Expectation Maximization (as discussed e.g. in \cite{van_zwol_predictive_2024}). In public PCN implementations \cite{tschantz_infer-activelypypc_2023}, \eqref{eq:activity_FNN} was already used during testing, a practice we have now given a more principled justification.
% For practical purposes, the conclusion is the same, however: one should simply use \eqref{eq:activity_FNN} during inference in a PCN since it is more computationally efficient.

A corollary of this proof is that since FNNs are universal function approximators \cite{cybenko_approximation_1989,prince_understanding_2023}, so are PCNs. This fact, although probably widely believed in the PC community, did not yet have a rigorous justification to the author's best knowledge. We also remark that by changing the definition of $a_i^\ell$ in the FNN, and correspondingly changing $\mu_i^\ell$ in the PCN, to e.g. a convolutional prediction \cite{van_zwol_predictive_2024} or skip connections, the above result may be trivially extended to any hierarchical structure, i.e. not just MLPs. 

\subsection{PCNs are subsets of PCGs}
We now define a PCG, a PCN generalization introduced by \cite{salvatori_learning_2022}. {This work nicely illustrated how different network topologies, such as hierarchical networks, could be obtained by masking the PCG weight matrix. However, it was not discussed how the resulting energies and dynamics (activity and learning rules) relate to PCNs. Here, we provide a rigorous proof that a PCG with a certain choice of weight matrix is equivalent to a PCN.} This enables the statement that PCGs are supersets of PCNs (as was mentioned in \cite{van_zwol_predictive_2024}, but not yet proven). For clarity, we use tildes and Greek indices for PCG nodes and weights, and Latin indices for the PCN.
% \begin{mdframed}
    \begin{definition}
A PCG is defined by a set of nodes $a_\alpha\in\mathbb{R}$, $\alpha=1,...,N$, and an \emph{energy} $E_G=\sum_\alpha\epsilon_\alpha^2$, with $\epsilon_\alpha = a_\alpha-\mu_\alpha$ and $\mu_\alpha=f\Big(\sum_{\beta=1}^N w_{\alpha\beta}a_\beta\Big)$, where $w_{\alpha\beta}^\ell \in \mathbb{R}$ are the weights, and $f$ is a non-linear element-wise activation function. Its \emph{activity rule} is:
\begin{align}
 % \forall i, 
   % \forall i, 
   n_x <\alpha \leq  N-n_y:\quad &\hat{a}_\alpha=\underset{{a_\alpha}}{\operatorname{argmin}}\,E_G, & \emph{(training)}    \label{eq:activity_PCG_training}\\
   n_x < \alpha \leq N :\quad &\hat{a}_\alpha=\underset{{a_\alpha}}{\operatorname{argmin}}\,E_G, & \emph{(testing)} \label{eq:activity_PCG_testing}
\end{align}
with $a_\alpha=x_\alpha$ for $0<\alpha\leq n_x$, and during training $a_\alpha=y_i$ for $N-n_y< \alpha\leq N,\,1\leq i\leq n_y$. The \emph{learning rule} is:
\begin{equation}
\begin{aligned}
  % 0< \alpha\leq n_x:\quad & a_\alpha=x_\alpha\\
  % N-n_y< \alpha\leq N:\quad & a_\alpha=y_\alpha\\
    \forall \alpha,\beta:\quad &\hat{w}_{\alpha\beta}=\underset{w_{\alpha\beta}}{\operatorname{argmin}}\,E_G & \emph{(training),}
\end{aligned}
\end{equation}
with $a_\alpha=x_\alpha$ for $0< \alpha\leq n_x$ and  $a_\alpha=y_i$ for $N-n_y< \alpha\leq N$ with $1\leq i\leq n_y$.
\end{definition}
% \end{mdframed}
% \begin{mdframed}
\begin{theorem}
\label{thm:thm2}
The PCG is a superset of the PCN through its weight matrix  $\widetilde{\bm{w}}$. With layers defined by $\{n_\ell\}_{\ell=0}^L$, with $n_\ell\in\mathbb{N}_+$, $\sum_\ell n_\ell=N$, partitioning
% \begin{itemize}
%     \item
nodes $\widetilde{\bm{a}}$ into layers $\bm{a}^\ell\in\mathbb{R}^{n_\ell}$; 
and weights $\widetilde{\bm{w}}$ into block matrices $\widetilde{\bm{w}}^{\ell k}\in\mathbb{R}^{n_\ell\times n_k}$, the choice  $\widetilde{\bm w}^{\ell k} = \widetilde{\bm w}^{\ell k}\delta_{k\ell-1}\equiv\bm w^{\ell-1}$ (with $\bm w^{\ell}$ the PCN weight matrix) implies that their objective functions are equivalent:{
\begin{equation}
    E_G=E_N+C
\end{equation}
where $C$ is a constant. Moreover, their activity rules are equivalent, both during training and testing:
\begin{equation}
\begin{aligned}
 % 0<\ell <L, &
 \quad \operatorname{argmin}_{a_i^\ell} E_N = \operatorname{argmin}_{a_\alpha} E_G  &\quad \emph{(training),}\\
  % 0<\ell\leq L,&
  \quad \operatorname{argmin}_{a_i^\ell} E_N= \operatorname{argmin}_{a_\alpha} E_G &\quad \emph{(testing),}
\end{aligned}
\end{equation}
and their learning rules are equivalent:
 \begin{equation}
\begin{aligned}
% 0\leq\ell <L,&\quad
\operatorname{argmin}_{w_{ij}^\ell} E_N = \operatorname{argmin}_{w_{\alpha\beta}} E_G  &\quad \emph{(training).}
\end{aligned}
\end{equation}
}Other choices of non-zero block matrices leads to skip connections, backward (skip) connections and lateral connections.
\end{theorem}
% \end{mdframed}
The proof is given in \ref{sec:proof_thm2}.
% and \ref{sec:single_node_width} proves the result for layers with one node for pedagogic purposes. 
It involves defining the exact mapping of nodes and weights, for which we again provide a brief intuition. From  $\{n_\ell\}_{\ell=0}^L$, one may define a partitioning of node indices as:
\begin{equation}
\begin{aligned}
I&=\{ \underbrace{1,\dots,n_0}_{\text{layer 0}},\underbrace{n_0+1,\dots,n_0+n_1}_{\text{layer 1}},\dots,\underbrace{s_{L-1}+1,\dots, s_L}_{\text{layer }L}\},\\
\end{aligned}
\end{equation}
where $s_\ell=\sum_{k=0}^\ell n_k$. With this, the PCG weights $\widetilde{\bm{w}}$ may be partitioned into block matrices $\widetilde{\bm w}^{\ell k}$, containing weights from layer $k$ to layer $\ell$, which yields:
\begin{equation}
\widetilde{\bm{w}}=\begin{bmatrix}
 \widetilde{\bm  w}^{00} & \widetilde{\bm w}^{01} &  \dots & &  \widetilde{\bm w}^{0L} \\
 \widetilde{\bm w}^{10} & \widetilde{\bm w}^{11} &        &  & \vdots  \\ 
 \vdots            &              &  \ddots       &    &    \\
 \widetilde{\bm w}^{L0}& \dots       &        &   & \widetilde{\bm w}^{LL}    \\
\end{bmatrix}=\begin{bmatrix}
\bm 0      & \bm 0      &  \dots   & &\bm 0     \\
{{\bm w}}^{0} & & & &\vdots\\ 
& {{\bm w}}^{1} &  &  &  \\
\vdots& &  \ddots&     & \\
\bm 0&\dots &       &{{\bm w}}^{L-1}    & \bm 0  \\
\end{bmatrix}\,.
\label{eq:bigmatrix2a_1}
\end{equation}
Here, in the last equality we set the hierarchical structure using $\widetilde{\bm w}^{\ell k} = \widetilde{\bm w}^{\ell k}\delta_{k\ell-1}\equiv\bm w^{\ell-1}$, with $\delta_{ij}$ the Kronecker delta, identifying the PCN weight matrices $\bm w^\ell$. The proof then proceeds to show how this leads to {the same energy (up to a constant), and identical dynamics (activity and learning rules), both during training and testing.}

The main corollary of our theorem is that during testing, PCGs are also a superset of FNNs, by virtue of theorem \ref{thm:thm1}. As a consequence, PCGs are also universal function approximators when a hierarchical structure \eqref{eq:bigmatrix2a_1} is chosen. Whether PCGs with additional connections can also approximate any function is left for future work.

{The superset nature of PCGs is illustrated in fig. \ref{fig:big_matrix}, complementing fig. 4 in \cite{salvatori_learning_2022}. A partitioned weight matrix is shown, where block matrices are color-mapped to visualizations of connections allowed by PCGs: forward (skip) connections like FNNs, and additionally backward (skip) connections and lateral connections.}

\begin{figure}[]
\includegraphics[width=\textwidth]{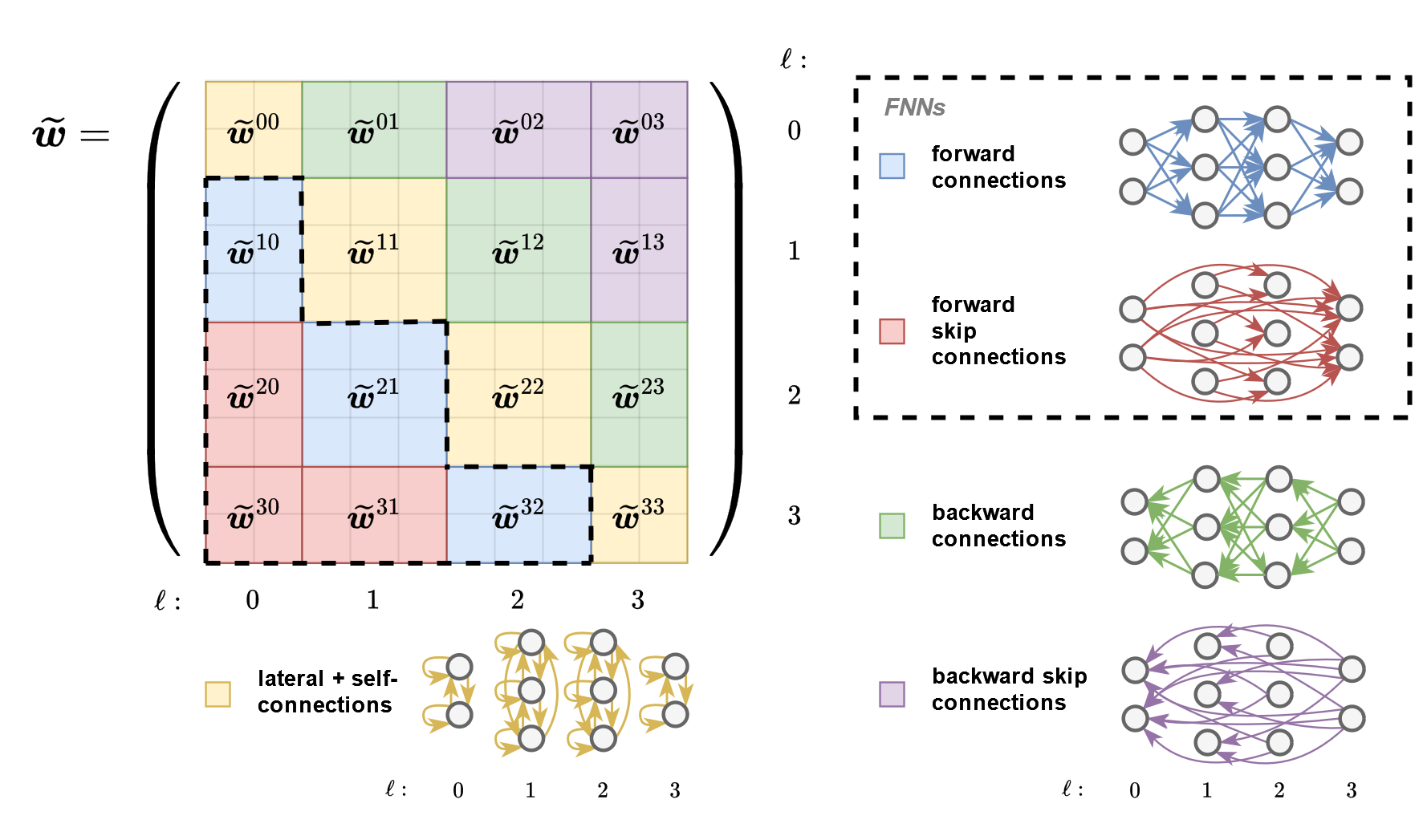}
  \caption{A PCG weight matrix partitioned into block matrices ($N=10$ nodes, partitioned by 4 layers with 2, 3, 3, and 2 nodes/layer, respectively). In traditional ANNs trained with BP, only feedforward connections (blue and red blocks) are trainable, whereas the full matrix is trainable with IL. {This figure extends fig. 4 in \cite{salvatori_learning_2022}.}}
\label{fig:big_matrix}
\end{figure}

\section{Discussion}
We have proven how PCGs form a mathematical superset of FNNs. By virtue of using IL, a more biologically plausible alternative to BP \cite{golkar_constrained_2022,song_inferring_2024}, they generalize  FNNs to {arbitrary} graphs. {This clarifies the relation between traditional neural networks and PCNs, the latter not being widely known in the broader ML community (as also recently argued by \cite{van_zwol_predictive_2024}).} 

{In the process, we showed that PCNs are \textit{equivalent} to FNNs during testing -- a subtle change from earlier work showing that PCNs \textit{converge} to FNNs. As such, this provides a rigorous argument for why the universal approximation theorem is applicable to PCNs, a statement which does not appear yet in the PC literature to the author's knowledge.}

PCGs allow a large new set of structures untrainable by BP, as was already argued by \cite{salvatori_learning_2022}. Our results rigorously show how these models includes FNNs as a special case. Moreover, fig. \ref{fig:big_matrix} clarifies how \textit{skip connections} can be seen as a part of the PCG weight matrix $\widetilde{\bm w}$ (extending fig. 4 in \cite{salvatori_learning_2022}). Such connections are the main innovation in ResNets \cite{he_deep_2016}, and are well-known provide great benefits for many ML tasks compared to standard MLPs \cite{zagoruyko_wide_2016,orhan_skip_2018}. Understanding these as a part of $\widetilde{\bm w}$, then, begs the question of whether the remainder of $\widetilde{\bm w}$ -- that is, backward (skip) connections, lateral connections and self-connections --  does, by analogy, also bring benefits. Future work should investigate this. So far, only all-to-all connected PCGs (the full matrix except self-connections) have been studied empirically \cite{salvatori_learning_2022}: these appear not to perform as well as hierarchical PCNs/FNNs, but they do outperform other all-to-all connected networks for classification on MNIST by a large margin (12-35\% better than Boltzmann machines and Hopfield networks for three datasets), which is encouraging. 

A practical limitation of current PCG implementations is that using non-feedforward connections comes at a computational cost, because gradient-based inference of nodes is relatively expensive. In FNNs, the time complexity of testing one batch is  $\mathcal{O}(LM)$, where $M = \max(\{ n^\ell n^{\ell+1}\}_{\ell})$ is the number of weights in the largest block matrix in $\widetilde{\bm{w}}$ \cite{van_zwol_predictive_2024}. In contrast, in PCGs this is $\mathcal{O}(N^2T)$ where $T$ is the number of inference steps. If sparsity of the PCG weight matrix is leveraged, this reduces to $\mathcal{O}(dNT)$ where $d$ is the number of nonzero weights. For an FNN with a comparable number of weights, one has $d\approx LM$. I.e., testing one batch is a factor $NT$ slower in a PCG. This is not necessarily undesirable, however, since increased testing time could potentially be compensated by other favorable properties of the training algorithm and topology used. {For further discussion of practical issues in PCNs and PCGs, we refer to \cite{van_zwol_predictive_2024}.}

As also mentioned in \cite{van_zwol_predictive_2024}, we emphasize that the non-feedforward connections in $\widetilde{\bm w}$ introduce a notion of \textit{recurrence} {distinct} from that in recurrent neural networks (RNNs), cf. fig. \ref{fig:big_matrix}. RNNs, made specifically to handle sequential data, have recurrence with respect to `data time', characterized by weight sharing. In PCGs, by contrast, recurrence is in `inference time', like in Hopfield networks \cite{mackay_information_2003}. This distinction was not yet highlighted in \cite{salvatori_learning_2022}.

Finally, our work underscores the value of theoretical work and rigorous mathematical arguments for studying PCNs, which is still limited in the literature -- a point also recently made by \cite{frieder_bad_2024, van_zwol_predictive_2024}. By providing theoretical justification, such work can usefully guide and constrain future experimental studies of PCNs.

\begin{ack}
% The author thanks partners, members from the HONDA project between the Honda Research Institute in Japan and Utrecht University in the Netherlands, for supporting this research. 
The author is grateful to Egon L. van den Broek and Ro Jefferson for valuable discussions and feedback on the manuscript. He also thanks Lukas Arts for helpful discussions. 
\end{ack}

{\small  % Begin small font size
\bibliographystyle{plain}
\bibliography{references}
}
% \section*{References}

% References follow the acknowledgments in the camera-ready paper. Use unnumbered first-level heading for
% the references. Any choice of citation style is acceptable as long as you are
% consistent. It is permissible to reduce the font size to \verb+small+ (9 point)
% when listing the references.
% Note that the Reference section does not count towards the page limit.
% \medskip

% {
% \small

% [1] Alexander, J.A.\ \& Mozer, M.C.\ (1995) Template-based algorithms for
% connectionist rule extraction. In G.\ Tesauro, D.S.\ Touretzky and T.K.\ Leen
% (eds.), {\it Advances in Neural Information Processing Systems 7},
% pp.\ 609--616. Cambridge, MA: MIT Press.

% [2] Bower, J.M.\ \& Beeman, D.\ (1995) {\it The Book of GENESIS: Exploring
%   Realistic Neural Models with the GEneral NEural SImulation System.}  New York:
% TELOS/Springer--Verlag.

% [3] Hasselmo, M.E., Schnell, E.\ \& Barkai, E.\ (1995) Dynamics of learning and
% recall at excitatory recurrent synapses and cholinergic modulation in rat
% hippocampal region CA3. {\it Journal of Neuroscience} {\bf 15}(7):5249-5262.
% }

%%%%%%%%%%%%%%%%%%%%%%%%%%%%%%%%%%%%%%%%%%%%%%%%%%%%%%%%%%%%

\appendix

\section{Proofs}\label{sec:appendix}
This appendix provides proofs for the two theorems in the main text.
\subsection{PCNs are FNNs during testing}\label{sec:proof_thm1}
This section proves theorem \ref{thm:thm1}.
\begin{proof}
 We have to prove \eqref{eq:equivalence}, which we do using backwards induction. The induction hypothesis is, for some $k<L$:
\begin{equation}
    \forall i:\quad {a}_i^{k}=f\Big(\sum_j w^{k-1}_{ij}a^{k-1}_j\Big)\iff \epsilon_i^k=0~.
     \label{eq:IH}
\end{equation}
The base case is $\ell=L$, by \eqref{eq:SOE_compact_2}:
\begin{equation}
\begin{aligned}
 \forall i:\quad {\epsilon}_i^L = 0  \iff {a}_i^{L}&=f\Big(\sum_j w^{L-1}_{ij}a^{L-1}_j\Big)~.
\end{aligned}
% \label{eq:dEdaL}
\end{equation}
Then, by \eqref{eq:SOE_compact_1}, for $\ell=k$:
\begin{equation}
\forall i:\quad\epsilon^{k-1}_i-\sum_{j=1}^{n_{k-1}} w^{k-1}_{ji}\epsilon_j^{k}f'\Big(\sum_m w^{k-1}_{jm}a^{k-1}_m\Big)=0 
\end{equation}
Then, by the induction hypothesis \eqref{eq:IH} $\epsilon_j^{k}=0$, so:
\begin{equation}
\forall i:\quad \epsilon^{k-1}_i=0.
\end{equation}
So, by induction the hypothesis \eqref{eq:IH} is proven for all $\ell\leq L$. 
\end{proof}

\subsection{PCNs are subsets of PCGs}\label{sec:proof_thm2}
This section proves theorem \ref{thm:thm2}.
% , first generally and section \ref{sec:proof_thm2} for single-node width, for pedagogic purposes.
% \subsubsection{General proof}
\begin{proof}
The PCN energy is:
\begin{equation}
   E_N = \sum_{\ell=0}^L\sum_{i=1}^{n_\ell} \Big[{a}_i^\ell-f\Big(\sum_{j=1}^{n_{\ell-1}}{w}_{ij}^{\ell-1}{a}^{\ell-1}_j\Big)\Big]^2,
\end{equation}
where ${\bm w}^{\ell} \in \mathbb{R}^{n_\ell \times n_{\ell+1}}$ are  matrices, and ${\bm a}^0=\bm x,\, {\bm a}^L=\bm y$. 
% The proof here proceeds in the same way, except that we must partition the PCG weight matrix in submatrices of various sizes, defined by the collection of nodes in each layer, $\{n_\ell\}_{\ell=0}^L$. 
It is useful to define \textit{the sum of node indices up to layer} $\ell$ as $s_\ell=\sum_{k=0}^\ell n_k$ 
where we additionally define $s_{-1} = 0$. Then, partition the PCG node indices $i$ as follows (cf. fig. \ref{fig:mapping}):
\begin{equation}
\begin{aligned}
I&=\{ \underbrace{1,\dots,n_0}_{I_0},\underbrace{n_0+1,\dots,n_0+n_1}_{I_1},\dots,\underbrace{s_{L-1}+1,\dots, s_L}_{I_L}\}.\\
\end{aligned}
\label{eq:partition}
\end{equation}
\begin{figure}[]
\includegraphics[width=\textwidth]{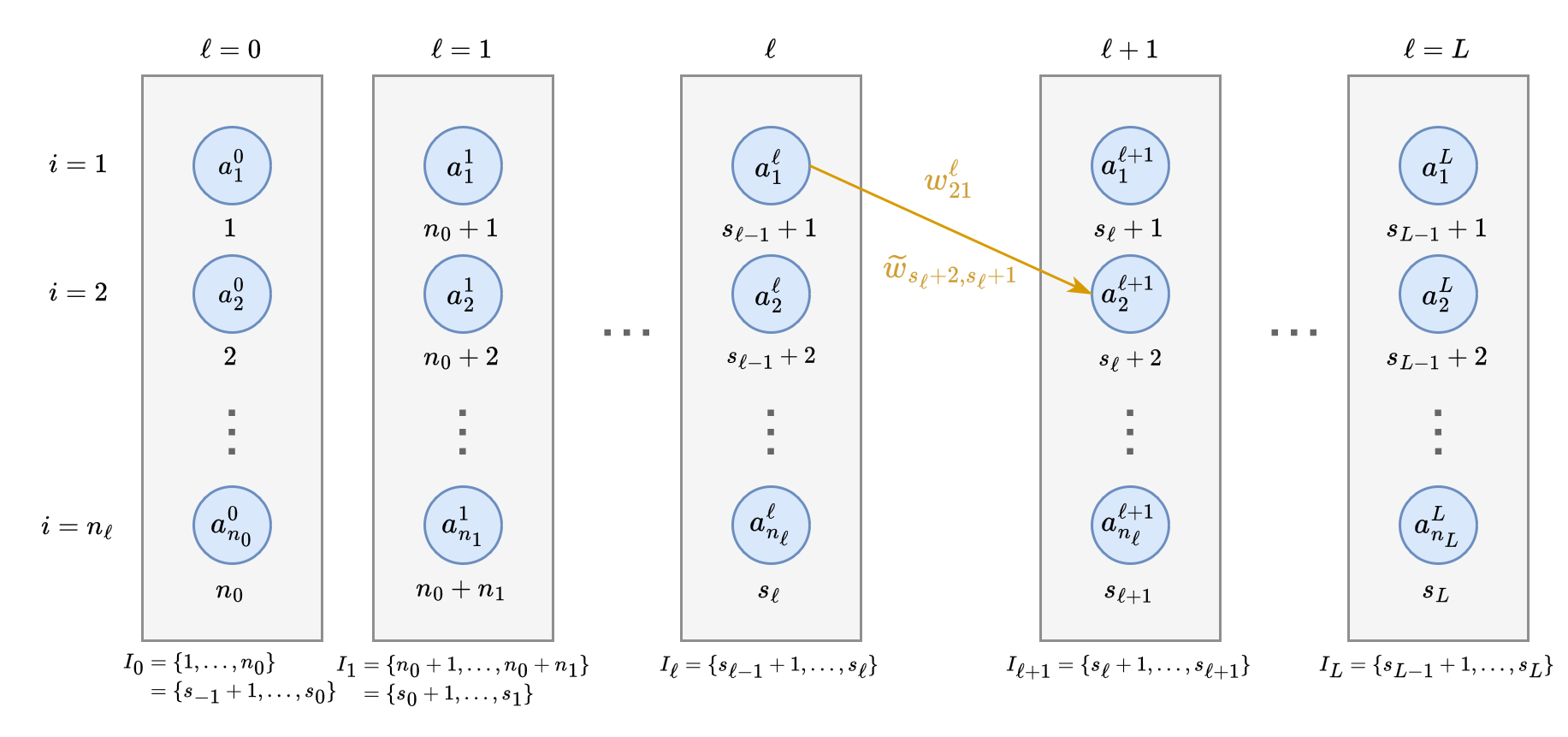}
  \caption{Mapping between PCN indices (blue nodes) and PCG indices (below nodes). Note that layers may have different widths $n_\ell$. $i$ denotes PCN index within a layer, and $I_\ell$ is defined by \eqref{eq:index_subsets}. One weight is shown in orange using PCN indices (top) and PCG indices (bottom), cf. \eqref{eq:hierarchical_pcg}.}
\label{fig:mapping}
\end{figure}
We defined $I_\ell$ as \textit{the node indices in layer} $\ell$, i.e.:
\begin{equation}
% I_\ell = \left\{ \sum_{k=0}^{\ell-1} n_k + 1, \sum_{k=0}^{\ell-1} n_k + 2, \dots, \sum_{k=0}^{\ell} n_k \right\},
I_\ell = \left\{ s_{\ell-1} + 1,s_{\ell-1} + 2, \dots, s_\ell \right\}.
\label{eq:index_subsets}
\end{equation}
Thus we have, if $\alpha\in I_\ell$, the corresponding PCN index $i(\ell) = \alpha-s_\ell$, meaning we can map PCG nodes to PCN nodes as follows:
\begin{equation}
    \widetilde{a}_{\alpha}={a}^\ell_i,
\label{eq:node_mapping}
\end{equation}
such that sums over these nodes may be written as
\begin{equation}
 \sum_{\alpha \in I_\ell} = \sum_{i=1}^{n_\ell}.
\label{eq:sum_mapping}
\end{equation}
The partition \eqref{eq:partition} also allows us to partition the weight matrix into block matrices. Defining  $\widetilde{\bm w}^{\ell k} \in \mathbb{R}^{n_\ell\times n_k}$ the full weight matrix becomes:
\begin{equation}
\widetilde{\bm{w}}=\begin{bmatrix}
 \widetilde{\bm  w}^{00} & \widetilde{\bm w}^{01} &  \dots & &  \widetilde{\bm w}^{0L} \\
 \widetilde{\bm w}^{10} & \widetilde{\bm w}^{11} &        &  & \vdots  \\ 
 \vdots            &              &  \ddots       &    &    \\
 \widetilde{\bm w}^{L0}& \dots       &        &   & \widetilde{\bm w}^{LL}    \\
\end{bmatrix}.
\label{eq:bigmatrix2a}
\end{equation}
% Observe here that each row and column is structured as \eqref{eq:partition}. 
Each matrix $\widetilde{\bm w}^{\ell k}$ contains weights from layer $\ell$ to layer $k$, cf. fig. \ref{fig:big_matrix}.
% \footnote{Unlike in traditional ANNs, one can have non-hierarchical connections, e.g. from layer 3 to layer 1 or self-connections in layer 2.} 
To obtain a hierarchical structure, one sets $\widetilde{\bm w}^{\ell k} = \widetilde{\bm w}^{\ell k}\delta_{k\ell-1}\equiv\bm w^{\ell-1}$, where $\delta_{ij}$ is the Kronecker delta, and we have found a mapping with the PCN weight matrix $\bm w^\ell$. Equivalently, this can be stated as follows: if $\alpha \in I_\ell$, then
\begin{equation}
    \widetilde{w}_{\alpha\beta}=
    \begin{dcases}
    {w}^\ell_{ij} & \text{if}\quad  \beta \in I_{\ell-1}\\
    0 & \text{else}
    \end{dcases}
    \label{eq:hierarchical_pcg}
\end{equation}
where $i=\alpha-s_\ell,\,j=\beta-s_\ell$ (cf. fig. \ref{fig:mapping}). This results in:
\begin{equation}
\widetilde{\bm{w}}=\begin{bmatrix}
\bm 0      & \bm 0      &  \dots & & \bm 0     \\
 \widetilde{\bm w}^{10} &        &        &  & \vdots  \\ 
       & \widetilde{\bm w}^{21} &        &    &    \\
 \vdots      &        &  \ddots&     &   \\
\bm 0       & \dots       &        & \widetilde{\bm w}^{LL-1}   & \bm 0   \\
\end{bmatrix}=\begin{bmatrix}
\bm 0      & \bm 0      &  \dots   & &\bm 0     \\
{{\bm w}}^{0} & & & &\vdots\\ 
& {{\bm w}}^{1} &  &  &  \\
\vdots& &  \ddots&     & \\
\bm 0&\dots &       &{{\bm w}}^{L-1}    & \bm 0  \\
\end{bmatrix}
\label{eq:bigmatrix2}
\end{equation}
% Observe that this form is the same as \eqref{eq:bigmatrix} but matrix entries have become block matrices.
Now, the PCG energy is:
\begin{equation}
E_G = \sum_{\alpha=1}^N\Big[\widetilde{a}_\alpha-f\Big(\sum_{\beta=1}^N \widetilde{w}_{\alpha\beta}\widetilde{a}_\beta\Big)\Big]^2  \label{eq:E_G}
\end{equation}
We decompose the outer and inner sum according to \eqref{eq:partition}, which we can write using \eqref{eq:index_subsets}:
$$
\begin{aligned}
% \sum_{i=0}^N  &= \sum_{i=0}^{n_0} +  \sum_{i=n_0+1}^{n_0+n_1} +\dots +  \sum_{i=\sum_{\ell=0}^{L-1}n_{\ell}}^N\\
\sum_{\alpha=1}^N &=\sum_{\alpha\in I_0} +  \sum_{\alpha\in I_1} +\dots +  \sum_{\alpha\in I_L}\\
% &= \sum_{\alpha=1}^n_1 +  \sum_{\alpha=1}^n_2 +\dots +  \sum_{\alpha=1}^n_L
\end{aligned}
$$
which yields
\begin{equation}
\begin{aligned}
E_G &= \sum_{\alpha\in I_0}\Big[\widetilde{a}_\alpha- f\Big(\sum_{\beta\in I_0} \widetilde{w}_{\alpha \beta}\widetilde{a}_\beta +  \sum_{\beta\in I_1} \widetilde{w}_{\alpha \beta}\widetilde{a}_\beta +\dots +  \sum_{\beta \in I_L} \widetilde{w}_{\alpha \beta}\widetilde{a}_j\Big) \Big]^2\\
&+\sum_{ \alpha \in I_1}\Big[\widetilde{a}_\alpha- f\Big(\sum_{\beta\in I_0} \widetilde{w}_{\alpha \beta}\widetilde{a}_\beta +  \sum_{\beta\in I_1} \widetilde{w}_{\alpha\beta}\widetilde{a}_\beta +\dots +  \sum_{\beta\in I_L} \widetilde{w}_{\alpha \beta}\widetilde{a}_\beta\Big) \Big]^2\\
&\vdots\\
&+\sum_{\alpha \in I_L}\Big[\widetilde{a}_\alpha- f\Big(\sum_{\beta\in I_0} \widetilde{w}_{\alpha \beta}\widetilde{a}_\beta +  \sum_{\beta\in I_1} \widetilde{w}_{\alpha \beta}\widetilde{a}_\beta +\dots +  \sum_{\beta \in I_L} \widetilde{w}_{\alpha \beta}\widetilde{a}_\beta\Big) \Big]^2,\\
\end{aligned}
\end{equation}
which we have written out to illustrate the decomposition into the block matrices in \eqref{eq:bigmatrix2a}. If we then use the mappings \eqref{eq:node_mapping}, \eqref{eq:sum_mapping}, and \eqref{eq:hierarchical_pcg}, this yields:
\begin{equation}
\begin{aligned}
E_G &= {\sum_{i=1}^{n_0} \Big[{a}_i^0 - f(0) \Big]^2}\\
&\quad+\sum_{i=1}^{n_1} \Big[{a}_i^1 - f\Big(\sum_{j=1}^{n_0} {w}_{ij}^{0} {a}_j^0 \Big) \Big]^2+\dots
+\sum_{i=1}^{n_L} \Big[{a}_i^L - f\Big(\sum_{j=1}^{n_{L-1}} {w}_{ij}^{L-1} {a}_j^{L-1} \Big) \Big]^2\\
&={\sum_{i=1}^{n_0} \Big[{a}_i^0 - f(0) \Big]^2}+\sum_{\ell=1}^L \sum_{i=1}^{n_\ell} \Big[{a}_i^\ell - f\Big(\sum_{j=1}^{n_{\ell-1}} {w}_{ij}^{\ell-1} {a}_j^{\ell-1} \Big) \Big]^2\\
&=E_N + C
\end{aligned}
\end{equation}

{Observe that during both training and testing, the lowest layer is clamped to the data, i.e. one always has $\forall i, a^0_i=x^0$. Hence, the extra term $C=\sum_{i} [{a}_i^0 - f(0) ]^2$ in $E_G$ can be considered constant with respect to both the activity and learning rules, since one has  (using the mappings \eqref{eq:node_mapping}, \eqref{eq:sum_mapping}, and \eqref{eq:hierarchical_pcg}):
\begin{equation}
\begin{aligned}
 0<\ell <L, &\quad \operatorname{argmin}_{a_i^\ell} E_N = \operatorname{argmin}_{a_\alpha} E_G \\
 0\leq\ell <L,&\quad \operatorname{argmin}_{w_{ij}^\ell} E_N = \operatorname{argmin}_{w_{\alpha\beta}} E_G.
\end{aligned}
\end{equation}
during training, and similarly during testing:
 \begin{equation}
\begin{aligned}
0< \ell\leq L,\quad \operatorname{argmin}_{a_i^\ell} E_N &= \operatorname{argmin}_{a_\alpha} E_G.
\end{aligned}
\end{equation}
In sum, both during training and testing, the dynamics of a PCG \textit{with weight matrix} \eqref{eq:bigmatrix2}, is exactly equal to the dynamics of the PCN, with equal energies up to a constant. \footnote{The mapping in the last layer may be checked by observing that since $n_L=n_y$ by assumption, one has $y_i=a^L_i=a_\alpha$ where $1\leq i<n_y$ and $N-n_y<\alpha \leq N \iff \sum_{k=0}^{L-1} n_k<\alpha \leq \sum_{k=0}^{L} n_k \iff \alpha \in I_L$.} As a consequence, since other choices of weight matrix leads to other architectures (e.g. all-to-all connectivity) \cite{salvatori_learning_2022}, PCGs define a superset of PCNs.\footnote{As a practical note, we mention that in implementations weight updates are typically calculated for all $i,j$ automatically using matrix multiplication. This means one should ensure weights are appropriately set to zero according to \eqref{eq:bigmatrix2} to keep the desired structure.}
}
% \qed
\end{proof}

\section{Update Rules}\label{sec:appendix_derivations}
Here we derive the gradient-based updates for the activity rule and the learning rule. To clarify the relation to other works, we do this for the two main conventions found in the literature: 
\begin{itemize}
    \item `Matrix-Activation', i.e. $\mu_i = f\left(\sum_j w_{ij}a_j\right)$. This is the convention most often considered for FNNs, used in the main text.
    \item `Activation-Matrix', i.e. $\mu_i = \sum_j w_{ij}f\left(a_j\right)$. This is often considered in the PC literature, e.g. \cite{song_can_2020,salvatori_learning_2022}.
\end{itemize}

We note that for the latter convention, the proofs above remains the same, with one slight exception in the second proof. In the final step, the constant term in $E_G$ becomes $C=\sum_{i}(a_i^0)^2$, since the non-linearities $f$ are multiplied by zero.

\subsection{PCN, matrix-activation}
% \subsubsection{PCN.}
Activations (eq. \ref{eq:SOE_compact_1}):
$$
\begin{aligned}
 \pdv{E}{a^\ell_i}&= \frac{1}{2}\sum_{k,j} 2\epsilon_j^k\Big[\delta^{k\ell}\delta_{ji}-f'\Big(\sum_{m}w^{k-1}_{jm}a^{k-1}_m\Big)\sum_m w_{jm}^{k-1} \delta^{k-1\ell}\delta_{mi}\Big]\\
&=\epsilon_i^\ell-\sum_j\epsilon_j^{\ell+1} f'\Big(\sum_{m}w^{\ell}_{jm}a^{\ell}_m\Big)w^\ell_{ji}
\end{aligned}
$$
where $\delta_{ij}$ is the Kronecker delta.
% $$
% \begin{aligned}
% \Delta a^\ell_i&= -\gamma\left[{\epsilon}_i^\ell-\sum_{j=1}^{n_{\ell}} {w}_{ji}^\ell {\epsilon}_j^{\ell+ 1}f'\left(\sum_{k=1}^{n_\ell} {w}_{jk}^\ell a_k^\ell \right)\right],   \\
% \Delta \bm a^\ell&=  -\gamma\Big(\bm{\epsilon}^\ell-\left(\bm{w}^\ell \right)^T \left[\bm{\epsilon}^{\ell+ 1}\odot f'\left( \bm{w}^\ell \bm{a}^\ell \right) \right]\Big)  
% \end{aligned}
% $$
% Note that the element-wise product requires brackets here, because we have to sum also over the $f'(.)$ factor. This is in contrast with the convention below. 
For the weights:
$$
\begin{aligned}
\pdv{E}{w^\ell_{ab}} &=-\frac{1}{2}\sum_{k,i}2\epsilon^k_if'\Big(\sum_j w^{k-1}_{ij}a_j^{k-1} \Big)'\sum_j x_j^{k-1}\delta_{ia}\delta_{jb}\delta^{k-1\ell}\\
&=-\epsilon_a^{\ell+1}f'\Big(\sum_j w^{\ell}_{aj}a_j^{\ell} \Big)a_b^\ell.
\end{aligned}
$$
% Which gives the update rules:
% $$
% \begin{aligned}
% \quad \Delta {w}^\ell_{ij}&= \alpha {\epsilon}^{\ell+1}_i f'\Big(\sum_{m}{w}^\ell_{im}a^\ell_m \Big){a}^\ell_j,\\
% \Delta \bm w^\ell&= \alpha [\bm{\epsilon}^{l+1}\odot f'\left(\bm{w}^\ell\bm{a}^\ell\right)](\bm{a}^\ell)^T
% \end{aligned}
% $$
% Here, also note that the brackets are key, because if the matrix product is done first, doing the element-wise product after will instead give $\epsilon^{\ell+1}_{\textcolor{red}{j}}$ instead of $\epsilon^{\ell+1}_{\textcolor{red}{i}}$. Also note here that one has $f'(.)$ here instead of $f(.)$ in the alternative convention below.

% \subsubsection{PCG}
% This is the same derivation as above, but leaving out all the layer indices. This gives:
% % $$
% % \begin{aligned}
% % \pdv{E}{a_i}&= \frac{1}{2}\sum_j 2\epsilon_j\Big[\delta_{ji}-f'\Big(\sum_{m}w_{jm}a_m\Big)\sum_m w_{jm}\delta_{mi}\Big]\\
% % &=\epsilon_i-\sum_j\epsilon_j f'\Big(\sum_{m}w_{jm}a_m\Big)w_{ji}\\
% % \end{aligned}
% % $$
% $$
% \begin{aligned}
% \Delta a_i&= -\gamma\left[{\epsilon}_i-\sum_{j=1}^{N} {w}_{ji}{\epsilon}_jf'\left(\sum_{m=1}^{N} {w}_{jm}a_m\right)\right]\\
% \Delta \bm a &=-\gamma\Big(\bm{\epsilon}-\bm{w} ^T \left[\bm{\epsilon}\odot f'\left( \bm{w} \bm{a}\right)\right]\Big),
% \end{aligned}
% $$
% And for weights:
% $$
% \begin{aligned}
% \Delta {w}_{ij}&= \alpha {\epsilon}_i f'\Big(\sum_{m}{w}_{im} {a}_m\Big)a_j\\ 
% \Delta \bm{w} &= \alpha [\bm{\epsilon}\odot f'\left(\bm{w}\bm{a} \right)]\bm{a}^T.
% \end{aligned}
% $$

\subsection{PCN, activation-matrix}
For completeness, we show the same derivation for the alternative convention.
% \subsubsection{PCN}
$$
\begin{aligned}
 \pdv{E}{a^\ell_i}&= \frac{1}{2}\sum_{k,j} 2\epsilon_j^k\Big[\delta^{k\ell}\delta_{ji}-\sum_{m}w^{k-1}_{jm}f'(a^{k-1}_m) \delta^{k-1\ell}\delta_{mi}\Big]\\
&=\epsilon_i^\ell-f'(a^{\ell}_i)\sum_j\epsilon_j^{\ell+1} w^{\ell}_{ji}
\end{aligned}
$$
% which gives the update rules:
% $$
% \begin{aligned}
% \Delta \bm a^\ell&= -\gamma\Big( \bm{\epsilon}^\ell- f'\left(\bm{a}^\ell\right)\odot[\left(\bm{w}^\ell\right)^T \bm{\epsilon}^{\ell+1}] \Big),\\
% \Delta a^\ell_i&= -\gamma\Big[{\epsilon^\ell_i}-f'({a^\ell_i})\sum_j{w}^{\ell}_{ji} \epsilon^{\ell+1}_j \Big].
% \end{aligned}
% $$
And for the weights:
$$
\begin{aligned}
\pdv{E}{w^\ell_{ab}} &=-\frac{1}{2}\sum_{k,i}2\epsilon^k_i \sum_j \delta_{ia}\delta_{jb}\delta^{k-1\ell} f(a_j^{k-1})\\
&=-\epsilon_a^{\ell+1}f(a^\ell_b).
\end{aligned}
$$
{
\section{Initialization}
We mention one final point relevant to using PCGs in practice. As mentioned, PCNs perform iterative gradient-based minimization of $E_N$ with respect to `hidden' (non-clamped) activations in layers $0< \ell < L$ during training. Unlike in FNNs, the result of inference (and hence, model performance) depends on \textit{initialization} of these nodes. Typically in the PC literature, this is done using a `feedforward pass' through the network, which typically appears to give better performance than random or zero initialization \cite{pinchetti_benchmarking_2024}. Writing activations at inference iteration $t$ as $a_i^\ell=a_i^\ell(t)$, the feedforward initialization is defined by:
$$
\begin{aligned}
\ell=0,\,\forall i:\quad &a^0_i(0)=x_i\\
\ell=L,\,\forall i:\quad &a^L_i(0)=y_i\\
0<\ell<L,\,\forall i: \quad&a_i^\ell(0)=\mu_i^\ell=f\Big(\sum_{j=1}^{n_{\ell-1}} w_{ij}^{\ell-1} a_i^{\ell-1}(0)\Big)
\end{aligned}
$$
For our work, this is relevant insofar as the scheme crucially depends on having a feedforward structure: for a {PCG} with both forward and non-forward connections, an analogous scheme is not obvious. This means that if one desires to get the results quoted in the literature for PCNs but using a PCG, this scheme has to be separately implemented (cf. e.g. the library implemented in \cite{van_zwol_predictive_2024}).
}

\end{document}